\documentclass[10pt,conference]{IEEEtran}
\makeatletter
 \def\blfootnote{\xdef\@thefnmark{}\@footnotetext}
 \makeatother
\usepackage[inner=0.675in,outer=0.675in,top=0.7in,bottom=1in,pdftex]{geometry}
\usepackage{amsthm}
\usepackage{bbm}
\usepackage[justification=centering]{caption}
\usepackage[usenames]{color}
\usepackage[noadjust]{cite}
\ifCLASSINFOpdf
   \usepackage[pdftex]{graphicx}
   \usepackage{epstopdf}
\else
  \usepackage[dvips]{graphicx}
\fi
\usepackage[cmex10]{amsmath}
\usepackage[]{amssymb}
\usepackage{algorithmic}
\usepackage{algorithm}
\usepackage{mdwmath}
\usepackage{mdwtab}
\usepackage{subfigure}
\usepackage{url}
\usepackage{color}
\usepackage[verbose]{wrapfig}

\allowdisplaybreaks
\newtheorem{theorem}{Theorem}
\newtheorem{Definition}{Definition}

\newtheorem{Remark}{Remark}

\newtheorem*{Example*}{Example}

\begin{document}
\title{Distributed Byzantine Tolerant Stochastic Gradient Descent in the Era of Big Data}
%\author{
%    \IEEEauthorblockN{Richeng Jin\IEEEauthorrefmark{1}, Xiaofan He\IEEEauthorrefmark{2}, Huaiyu Dai\IEEEauthorrefmark{1},
%    \IEEEauthorblockA{\IEEEauthorrefmark{1}Department of ECE, North Carolina State University\\
%    Emails: \{rjin2,hdai\}@ncsu.edu}
%    \IEEEauthorblockA{\IEEEauthorrefmark{2}Department of EE, Lamar University\\
%    Email: {xhe1@lamar.edu}}}}

\author{\IEEEauthorblockN{Richeng Jin}
\IEEEauthorblockA{\textit{Department of ECE} \\
\textit{North Carolina State University}\\
rjin2@ncsu.edu}
\and
\IEEEauthorblockN{Xiaofan He}
\IEEEauthorblockA{\textit{Electronic Information School} \\
\textit{Wuhan University}\\
xiaofanhe@whu.edu.cn}
\and
\IEEEauthorblockN{Huaiyu Dai}
\IEEEauthorblockA{\textit{Department of ECE} \\
\textit{North Carolina State University}\\
hdai@ncsu.edu}}

\maketitle
\begin{abstract}
The recent advances in sensor technologies and smart devices enable the collaborative collection of a sheer volume of data from multiple information sources. As a promising tool to efficiently extract useful information from such big data, machine learning has been pushed to the forefront and seen great success in a wide range of relevant areas such as computer vision, health care, and financial market analysis. To accommodate the large volume of data, there is a surge of interest in the design of distributed machine learning, among which stochastic gradient descent (SGD) is one of the mostly adopted methods. Nonetheless, distributed machine learning methods may be vulnerable to Byzantine attack, in which the adversary can deliberately share falsified information to disrupt the intended machine learning procedures. In this work, two asynchronous Byzantine tolerant SGD algorithms are proposed, in which the honest collaborative workers are assumed to store the model parameters derived from their own local data and use them as the ground truth. The proposed algorithms can deal with an arbitrary number of Byzantine attackers and are provably convergent. Simulation results based on a real-world dataset are presented to verify the theoretical results and demonstrate the effectiveness of the proposed algorithms.
\end{abstract}
{\blfootnote{This work was supported in part by the US National Science Foundation under Grants ECCS-1444009 and CNS-1824518.}}

\section{Introduction}
\noindent With the proliferation of sensors and smart devices, the past decade has witnessed the blowout growth in the size of the daily generated data. For example, according to \cite{dobre2014intelligent}, the world produces around 2.5 quintillion bytes of data per day in 2014. Also, as predicted by Cisco, there will be around 11.6 billion mobile devices by the year 2020 and a smartphone will generate 4.4 gigabytes data per month on average \cite{index2016cisco}. Facing such a data deluge, distributed machine learning is anticipated to play an essential role because of its ability to exploit the collective computation power of the local smart/sensing devices, thereby leading to enhanced big data analytics \cite{qiu2016survey}. Specifically, distributed machine learning mechanisms have several advantages over its centralized counterpart in big data related applications. Firstly, decentralization offers better scalability, and thus facilitates large-scale machine learning applications in practice. Secondly, it eliminates the burdensome process of moving the large amount of data from the distributed devices to a central unit \cite{jochems2016distributed} as well as the difficulty of storing the excessive amount of data in a single machine \cite{richtarik2016distributed}.

In the existing literature of distributed machine learning, the workers are usually assumed to be honest and perform perfectly well (i.e., not making any mistake in calculation and information transmission). However, in practice, some of the workers may share wrong information due to system malfunction or software bugs. Also, some of them may even be compromised by an adversary and deliberately share falsified information to mislead the other co-workers. As is shown in \cite{chen2017distributed} and \cite{blanchard2017machine}, even a single Byzantine worker can severely disrupt the convergence of distributed gradient descent algorithms. This problem becomes more critical in big data applications, since the large number of data collection devices (i.e., the workers in machine learning applications) and the sheer volume of the collected data make it extremely challenging, if not impossible, to ensure perfect trustworthiness in data sharing and processing.

There have been some recent works \cite{chen2017distributed,blanchard2017machine,xie2018generalized,yin2018byzantine,alistarh2018byzantine,damaskinos2018asynchronous,caodistributed} on Byzantine tolerant distributed machine learning algorithms, and most of them focus on stochastic gradient descent (SGD), which is one of the classic and widely adopted distributed machine learning algorithm with good scalability. However, most of them only consider the synchronous setting. This may lead to a waste of computation resources, since the workers with better computation capability have to wait for the other slower workers. In addition, most existing methods can deal with only a limited number of Byzantine workers. Moreover, they all assume a parameter server to coordinate the collaboration among the workers, which may be vulnerable to the single point of failure (SPOF).

To better accommodate the need of big data analytics, two asynchronous distributed Byzantine tolerant SGD algorithms that can deal with an arbitrary number of Byzantine workers are proposed in this work. Particularly, in the proposed algorithms, the workers are allowed to maintain their own local model parameters, which eliminates the need for a shared parameter server as in the existing literature. Also, in this setting, the workers do not need to wait for the latest broadcast model parameter from the parameter server and can proactively fetch the current learning results from the other (possibly Byzantine) co-workers at any time, thereby fulfilling asynchronous learning.\footnote{We note that the proposed algorithms indeed introduce some communication overhead. However, it can also be implemented with a parameter server which maintains and updates the local models for the workers. In this case, the communication overhead is similar to the existing methods in the literature mentioned above. Nonetheless, it may be vulnerable to SPOF.} The two proposed algorithms correspond to two different scenarios, respectively. In the first scenario, it is assumed that an \textit{upper-bound} $p$ of the number of Byzantine workers is known. To defend the Byzantine attack in this case, each worker takes an average of the $N-p$ model parameters that are closest to its own and then performs a gradient descent update step based on this average value. In the second scenario, no prior knowledge about the number of Byzantine workers is assumed. In this case, each worker first accepts the model parameters that potentially lead to lower empirical risk based on its evaluation over the local training samples. Then, the worker takes an average over the accepted model parameters and performs a gradient descent update step accordingly. Both of the proposed algorithms are provably convergent.

The remainder of this article is organized as follows. Section~\ref{preliminaries} reviews preliminaries and notations used in this work. The problem is formulated and presented in Section~\ref{ProblemFormulation}. The proposed algorithms are presented in Section~\ref{ProposedAlgorithms}. The effectiveness of the proposed algorithms is examined through simulations in Section~\ref{Numerical results}. Related works are discussed in Section \ref{Related Works}. Conclusions and future works are presented in Section~\ref{Conclusions and Future Works}.

\section{Preliminaries and Notations}\label{preliminaries}
\noindent In this section, we start by reviewing some important definitions. Suppose that there is a training data set $\mathcal{S}=\{(x_1,y_1),\cdots,(x_n,y_n)\}$ with $n$ training instances randomly sampled from a sample space $\mathcal{Z} = \mathcal{X} \times \mathcal{Y}$, where $\mathcal{X}$ is a space of feature vectors and $\mathcal{Y}$ is a label space. Let $\mathcal{W} \subseteq \mathbb{R}^{d}$ be a hypothesis space of the model parameter equipped with the standard inner product and 2-norm $||\cdot||$. Given a prediction model $h(w) \in \mathcal{F}: \mathcal{X} \rightarrow \mathcal{Y}$ which is parameterized by $w \in \mathcal{W}$, the goal is to learn a good model parameter $w$. The prediction accuracy is measured by a loss function $f: \mathcal{W} \times \mathcal{Z} \rightarrow \mathbb{R}$.Given a hypothesis $w \in \mathcal{W}$ and a training sample $(x_i,y_i) \in \mathcal{S}$, we have a loss $f(w,(x_i,y_i))$. SGD \cite{bottou2010large} is a commonly used optimization algorithm, which aims to minimize the empirical risk $F(w) = \frac{1}{n}\sum_{i=1}^{n}f(w,(x_i,y_i))$ over the training data set $\mathcal{S}$ of $n$ samples. For simplicity, let $f_{i}(w)=f(w,(x_i,y_i))$ for fixed $\mathcal{S}$. In each iteration, given a training sample $(x_t,y_t)$, SGD updates the hypothesis $w_t$ as follows:
\begin{equation}
w_{t+1} = G_{f_t,\eta_t} = w_{t} - \eta_{t}\nabla f_{t}(w_t),
\end{equation}
\noindent in which $\eta_{t}$ is the learning rate and $\nabla f_{t}(w_t) = \nabla f(w_t,(x_t,y_t))$ is the gradient.

To facilitate later discussion on convergence, some definitions related to the loss function are presented as follows.

%\begin{Definition}
%Let $g:\mathcal{W} \rightarrow \mathbb{R}$ be a function: $g$ is convex if for any $u,v \in \mathcal{W}$, $l(u) \geq g(v)+<\nabla g(v),u-v>$; $g$ is L-Lipschitz if for any $u,v \in \mathcal{W}$, $||g(u)-g(v)|| \leq L||u-v||$; $g$ is $\lambda-strongly$ convex if for any $u,v \in \mathcal{W}$, $g(u)\geq g(v)+<\nabla g(v),u-v> + \frac{\lambda}{2}||u-v||^2$.
%\end{Definition}

\begin{Definition}
Let $g:\mathcal{W} \rightarrow \mathbb{R}$ be a function:
\begin{itemize}
\item $g$ is convex if for any $u,v \in \mathcal{W}$, \\
$g(u) \geq g(v)+<\nabla g(v),u-v>$
\item $g$ is L-Lipschitz if for any $u,v \in \mathcal{W}$, \\
$||g(u)-g(v)|| \leq L||u-v||$
\item $g$ is $\lambda-strongly$ convex if for any $u,v \in \mathcal{W}$, \\
$g(u)\geq g(v)+<\nabla g(v),u-v> + \frac{\lambda}{2}||u-v||^2$
\end{itemize}
\end{Definition}
\vspace{0.1in}
\section{Problem Formulation}\label{ProblemFormulation}
\noindent In this work, a network $\mathcal{N}=\{1,\cdots,N\}$ consisting of $N$ collaborative workers, each storing a portion of a dataset, is considered. It is assumed that each worker $i$ stores and updates its own local model parameter $w^{i}_{t}$ which will be used for their own classification tasks. In addition, asynchronous update is assumed in this work, where each worker can start its next update step immediately once the previous step finishes. Particularly, as is frequently done in the literature (e.g., \cite{boyd2006randomized,jin2016scale}), the time steps are modeled as the ticking of local clocks governed by Poisson processes. It is assumed that each worker has a clock that ticks with a rate 1 Poisson process. Thus, the inter-tick times at each worker are rate 1 exponentials, independent across workers and over time. In addition, there is a master clock which ticks whenever a local processor clock ticks and the time is discretized according to the master clock ticks (since these are the only time that the local models are updated). In this sense, the master clock ticks according to a rate $N$ Poisson process and the local clock $i$ that causes each master clock tick is an independently and identically distributed (i.i.d.) random variable drawn from $\mathcal{N}$. At iteration step $t$ (when there have been $t-1$ total update steps for all the workers), a worker $i$ sends requests and fetches the local model parameters from all the other workers (i.e., $w^{j}_{t}, \forall j \in \mathcal{N}/\{i\}$) and updates its local models based on the shared model parameters and its local dataset. In this work, it is assumed that up to $p$ workers are Byzantine which behave arbitrarily and can share any information. Furthermore, the Byzantine workers are assumed to be aware of the local model parameters from the honest workers since they can also send requests to them. Let $\mathcal{B}$ and $\mathcal{H}$ denote the sets of Byzantine workers and honest workers, respectively. After fetching, worker $i$ will receive

\begin{equation}
w^{j}_{t} =
\begin{cases}
\hfill w^{j}_{t}, \hfill &\text{if $j \in \mathcal{H}$},\\
\hfill \text{any $w \in \mathcal{W}$}, \hfill &\text{if $j \in \mathcal{B}$.}
\end{cases}
\end{equation}

The goal of this work is to design robust SGD algorithms that can tolerant any number of Byzantine attackers (i.e., $p$ can be any integer in $[0,N-1]$).
\vspace{0.1in}
\section{Proposed Algorithms}\label{ProposedAlgorithms}
\noindent In this section, the proposed algorithms are presented. In particular, depending on whether the upper bound $p$ of the number of Byzantine workers is known or not, two scenarios are considered. It is assumed that whenever a worker responds to a request and sends its model parameter to others, the shared information will arrive on time. In this case, all the Byzantine workers will choose to share something upon requests, since they can be easily identified if the others fail to receive information from them.

\subsection{Scenario 1: $p$ is known}
\noindent The main steps of the proposed algorithm are given in Algorithm \ref{algorithm1}. The main idea is that the local model parameter stored by worker $i$ (i.e., $w^{i}_{t}$) can serve as the ground truth, based on which the received shared model parameters (i.e., $w^{j}_{t}, \forall j \in \mathcal{N}/\{i\}$) can be filtered. Specifically, given the upper bound $p$ of the number of Byzantine workers, accepting the $N-p-1$ model parameters which are closest to an honest worker's own model parameter will intuitively help filter out the wrong information shared by the Byzantine workers.
\begin{algorithm}
\caption{Byzantine Tolerant SGD Algorithm when $p$ is known}
\label{algorithm1}
\begin{algorithmic}
\STATE 1. Initialization: total number of workers: $N$, number of training data samples for each node: $M$, upper-bound of the number of Byzantine workers: $p$, each honest worker $i \in \mathcal{H}$ randomly initialize its model parameter $w_{0}^{i}$.
\STATE 2. for iteration $t=0,1,\cdots,T$ do
\STATE 3. ~~ if worker $i$ causes the master clock to tick:
\STATE 4. ~~~~worker $i$ sends requests and fetches the model parameters from all the other workers, then it accepts the $N-p-1$ model parameters which are closest to its own (i.e., the $N-p-1$ $w^{j}_{t}$ with the smallest $||w^{i}_{t}-w^{j}_{t}||$). Then worker $i$ takes an average over the accepted model parameters, and randomly samples a mini-batch of training samples $\mathcal{S}_{t}^{i}$ from its local dataset and performs one gradient descent step as follows:
\begin{equation}\label{TakeAverage}
  w_{t+\frac{1}{2}}^{i} = \frac{w^{i}_{t}+\sum_{j \in \mathcal{A}_{t}^{i}}w^{j}_{t}}{N-p},
\end{equation}
\begin{equation}\label{PerformSGD}
  w_{t+1}^{i} = w_{t+\frac{1}{2}}^{i}-\eta_t\nabla f_{\mathcal{S}_{t}^{i}}(w_{t+\frac{1}{2}}^{i}),
\end{equation}
in which $w_{t+\frac{1}{2}}^{i}$ is the average of its own and the accepted model parameters for worker $i$, $\mathcal{A}_{t}^{i}$ is the set of accepted workers, $\nabla f_{\mathcal{S}_{t}^{i}}(w_{t+\frac{1}{2}}^{i})=\frac{1}{|\mathcal{S}_{t}^{i}|}\sum_{m \in \mathcal{S}_{t}^{i}}\nabla f_{m}(w_{t+\frac{1}{2}}^{i})$ is the average gradient and $\eta_t$ is the learning rate at time $t$.
\STATE 5. ~~~~worker $i$ normalizes its own model parameter, i.e.,
\begin{equation}\label{Normalization}
  w_{t+1}^{i} = \frac{w_{t+1}^{i}}{||w_{t+1}^{i}||}
\end{equation}
\STATE 6.~~end if
\STATE 7.end for
\end{algorithmic}
\end{algorithm}

Note that since each worker $i$ compares the received information with its own local model parameter by measuring the distance between the model parameters, naturally one possible attack against the proposed mechanism for the Byzantine workers is to add some random noise to the model parameter (i.e., $w^{i}_{t}$) of the worker that sends the request and then send the resulting perturbed result back to the requester. Note that in order to pass the accepting condition in Algorithm 1 (i.e., step 4), the Byzantine workers tend to modify the requester's model parameter moderately (otherwise it will be filtered out). Therefore, in the following analysis on Algorithm \ref{algorithm1}, it is assumed that for any $w^{j}_{t}$ from worker $j$ accepted by worker $i$ at time $t$, it satisfies $w^{j}_{t}=w^{i}_{t}+\epsilon$, with $\mathbb{E}[||\epsilon||]=0$ and $\mathbb{E}[||\epsilon||^2] \leq \sigma^2$, in which $\sigma$ is a positive real number. Nonetheless, we note that the proposed algorithm works well on other types of attacks too.

\begin{theorem}\label{theorem1}
Suppose that the loss function $f$ is $\lambda$-strongly convex. At each iteration $t$, assume that worker $i$ can sample a random gradient $\nabla f_{\mathcal{S}_{t}^{i}}(w_{t+\frac{1}{2}}^{i})$ that satisfies $\mathbb{E}[\nabla f_{\mathcal{S}_{t}^{i}}(w_{t+\frac{1}{2}}^{i})]=\mathbb{E}[\nabla F(w_{t+\frac{1}{2}}^{i})]$ and $||\nabla f_{\mathcal{S}_{t}^{i}}(w)||^2 \leq G^2$ for any $w$ and $\mathcal{S}_{t}^{i}$. Then running Algorithm \ref{algorithm1} with the time model as described, with a constant step size $\eta$, we have
\begin{equation}
\begin{split}
&\sum_{l=1}^{N-p}\mathbb{E}[||w_{t+1}^{l}-w^{*}||^2] \leq (1-\frac{2\eta\lambda}{N-p})^{t}\sum_{l=1}^{N-p}\mathbb{E}[||w_{0}^{l}-w^{*}||^2]\\
&+\frac{(1-2\eta\lambda)\sigma^2+\eta^2G^2}{2\eta\lambda}.
\end{split}
\end{equation}

\end{theorem}
\begin{proof}
%Please see \cite{jin2018technicalreport}.
Without loss of generality, we assume that the first $N-p$ workers are honest while the last $p$ workers are Byzantine. Then we have
\begin{equation}\label{SumDistance}
\begin{split}
&\sum_{l=1}^{N-p}\mathbb{E}[||w_{t+1}^{l}-w^{*}||^2] = \frac{N-p-1}{N-p}\big[\sum_{l=1}^{N-p}\mathbb{E}[||w_{t}^{l}-w^{*}||^2]\big] + \\ &\frac{1}{N-p}\sum_{l=1}^{N-p}\big[\mathbb{E}[||w_{t+\frac{1}{2}}^{l}-\eta\nabla f_{\mathcal{S}_{t}^{l}}(w_{t+\frac{1}{2}}^{l})-w^{*}||^2]\big].
\end{split}
\end{equation}

For any $l\in \mathcal{H}$, we have

\begin{equation}
\begin{split}
&\mathbb{E}[||w_{t+\frac{1}{2}}^{l}-\eta\nabla f_{\mathcal{S}_{t}^{l}}(w_{t+\frac{1}{2}}^{l})-w^{*}||^2]\\
&= \mathbb{E}[||w_{t+\frac{1}{2}}^{l}-w^{*}||^2] + \eta^2\mathbb{E}[||\nabla f_{\mathcal{S}_{t}^{l}}(w_{t+\frac{1}{2}}^{l})||^2]\\
&- 2\eta\mathbb{E}[<w_{t+\frac{1}{2}}^{l}-w^{*},\nabla f_{\mathcal{S}_{t}^{l}}(w_{t+\frac{1}{2}}^{l})>].
\end{split}
\end{equation}
According to the strongly convexity of the loss function,

\begin{equation}
\begin{split}
&\mathbb{E}[<w_{t+\frac{1}{2}}^{l}-w^{*},\nabla f_{\mathcal{S}_{t}^{l}}(w_{t+\frac{1}{2}}^{l})>]\\
&= \mathbb{E}[<w_{t+\frac{1}{2}}^{l}-w^{*},\nabla F(w_{t+\frac{1}{2}}^{l})>] \geq \lambda \mathbb{E}[||w_{t+\frac{1}{2}}^{l}-w^{*}||^2].
\end{split}
\end{equation}
As a result,
\begin{equation}\label{10}
\begin{split}
&\mathbb{E}[||w_{t+\frac{1}{2}}^{l}-\eta\nabla f_{\mathcal{S}_{t}^{l}}(w_{t+\frac{1}{2}}^{l})-w^{*}||^2]\\
&\leq(1-2\eta\lambda)\mathbb{E}[||w_{t+\frac{1}{2}}^{l}-w^{*}||^2] + \eta^2\mathbb{E}[||\nabla f_{\mathcal{S}_{t}^{l}}(w_{t+\frac{1}{2}}^{l})||^2].\\
\end{split}
\end{equation}
Note that $w_{t+\frac{1}{2}}^{l} = \frac{w^{l}_{t}+\sum_{m \in \mathcal{A}_{t}^{l}}w^{m}_{t}}{N-p}$, we have
\begin{equation}\label{11}
\begin{split}
&\mathbb{E}[||w_{t+\frac{1}{2}}^{l}-w^{*}||^2] \leq \frac{1}{N-p}\big[\mathbb{E}[||w_{t}^{l}-w^{*}||^2] \\
&+\sum_{m \in \mathcal{A}_{t}^{l}}\mathbb{E}[||w_{t}^{m}-w^{*}||^2]\big] \leq \mathbb{E}[||w_{t}^{l}-w^{*}||^2] + \frac{N-p-1}{N-p}\sigma^2.
\end{split}
\end{equation}

Plugging (\ref{10}) and (\ref{11}) into (\ref{SumDistance}), we obtain

\begin{equation}
\begin{split}
&\sum_{l=1}^{N-p}\mathbb{E}[||w_{t+1}^{l}-w^{*}||^2] \leq \frac{N-p-1}{N-p}\big[\sum_{l=1}^{N-p}\mathbb{E}[||w_{t}^{l}-w^{*}||^2]\big] \\
&+\frac{1}{N-p}\sum_{l=1}^{N-p}\big[(1-2\eta\lambda)\mathbb{E}[||w_{t}^{l}-w^{*}||^2]\\
&+\frac{(1-2\eta\lambda)(N-p-1)}{N-p}\sigma^2+\eta^2G^2\big]\\
&\leq \sum_{l=1}^{N-p}(1-\frac{2\eta\lambda}{N-p})\mathbb{E}[||w_{t}^{l}-w^{*}||^2]+\frac{\eta^2G^2}{N-p}\\
&+\frac{(1-2\eta\lambda)}{(N-p)}\sigma^2.
\end{split}
\end{equation}

Therefore,
\begin{equation}
\begin{split}
&\sum_{l=1}^{N-p}\mathbb{E}[||w_{t+1}^{l}-w^{*}||^2] \leq (1-\frac{2\eta\lambda}{N-p})^{t}\sum_{l=1}^{N-p}\mathbb{E}[||w_{0}^{l}-w^{*}||^2]\\
&+\frac{(1-2\eta\lambda)\sigma^2+\eta^2G^2}{2\eta\lambda}.
\end{split}
\end{equation}
\end{proof}

\begin{Remark}
Theorem \ref{theorem1} indicates that the local model parameters converge to a ball around the optimal solution, whose radius is upper bounded by a variable depending on the noise added to the shared model parameters from the Byzantine workers.
\end{Remark}

\subsection{Scenario 2: $p$ is unknown}
\noindent Note that in practice, the upper bound of the number of Byzantine workers may not be available for the honest workers, an algorithm that does not require any prior knowledge about $p$ is developed in this subsection. The main steps of the proposed algorithm are given in Algorithm \ref{algorithm2}. The main difference with Algorithm \ref{algorithm1} is the conditions of accepting a shared model parameter. In this case, the filtering criteria in Algorithm \ref{algorithm1} cannot be used since $p$ in unknown. Therefore, (\ref{condition2}) is proposed to prevent the workers from accepting model parameters that are too far away from their local model parameters. However, the performance induced by condition (\ref{condition2}) depends on the threshold parameter $\delta$. If $\delta$ is too large, condition (\ref{condition2}) cannot filter out the Byzantine workers when the total number of iterations is limited. If $\delta$ is too small, all the legit workers may be filtered, which renders the collaboration ineffective. With such consideration, condition (\ref{condition1}) is proposed to further improve the performance of Algorithm 2, especially for large $\delta$. In particular, (\ref{condition1}) indicates that if worker $i$ performs a stochastic gradient update based on the shared model parameter $w_{t}^{j}$ from worker $j$, its own local parameter model $w_{t}^{i}$ is not in the direction of this update and therefore $w_{t}^{j}$ is supposed to be closer to the optimal $w^{*}$. In addition, (\ref{condition1}) is the sufficient condition for $\sum_{k \in \mathcal{S}_{t}^{i}}[f_{k}(w_{t}^{i})-f_{k}(w_{t}^{j})]\geq 0$ when the loss function $f$ is convex, which essentially means that the model parameter shared by worker $j$ is likely to be better than the local one.

The convergence of Algorithm \ref{algorithm2} is given as follows.
\begin{theorem}
Suppose that the loss function $f$ is $\lambda$-strongly convex with $L$-Lipschitz gradients. At each iteration $t$, assume that worker $i$ can sample a random gradient $\nabla f_{\mathcal{S}_{t}^{i}}(w_{t+\frac{1}{2}}^{i})$ that satisfies $\mathbb{E}[\nabla f_{\mathcal{S}_{t}^{i}}(w_{t+\frac{1}{2}}^{i})]=\mathbb{E}[\nabla F(w_{t+\frac{1}{2}}^{i})]$. Then running Algorithm \ref{algorithm2} with the time model as described, with a constant step size $0 \le \eta \leq \frac{2}{\lambda+L}$, we have
\begin{equation} \label{Contradiction}
\begin{split}
&\sum_{l=1}^{N-p}\mathbb{E}[||w_{t+1}^{l}-w^{*}||] \\
&\leq (1-\frac{\eta\lambda L}{(N-p)(\lambda+L)})^{t}\sum_{l=1}^{N-p}\mathbb{E}[||w_{0}^{l}-w^{*}||]\\
&+(1-\frac{\eta\lambda L}{\lambda+L})\frac{\delta}{(N-p)}\sum_{k=0}^{t}\frac{(1-\frac{\eta\lambda L}{(N-p)(\lambda+L)})^{(t-k)}}{k+1}.
\end{split}
\end{equation}
\end{theorem}

\begin{proof}
%Please see \cite{jin2018technicalreport}.
Without loss of generality, we assume that the first $N-p$ workers are honest while the last $p$ workers are Byzantine. Then, we have
\begin{equation}\label{SumDistance}
\begin{split}
&\sum_{l=1}^{N-p}\mathbb{E}[||w_{t+1}^{l}-w^{*}||] = \frac{N-p-1}{N-p}\big[\sum_{l=1}^{N-p}\mathbb{E}[||w_{t}^{l}-w^{*}||]\big] + \\ &\frac{1}{N-p}\sum_{l=1}^{N-p}\big[\mathbb{E}[||w_{t+\frac{1}{2}}^{l}-\eta\nabla f_{\mathcal{S}_{t}^{l}}(w_{t+\frac{1}{2}}^{l})-w^{*}||]\big].
\end{split}
\end{equation}

For any $l\in \mathcal{H}$, we have

\begin{equation}
\begin{split}
&\mathbb{E}[||w_{t+\frac{1}{2}}^{l}-\eta\nabla f_{\mathcal{S}_{t}^{l}}(w_{t+\frac{1}{2}}^{l})-w^{*}||^2]\\
&= \mathbb{E}[||w_{t+\frac{1}{2}}^{l}-w^{*}||^2] + \eta^2\mathbb{E}[||\nabla f_{\mathcal{S}_{t}^{l}}(w_{t+\frac{1}{2}}^{l})||^2]\\
&- 2\eta\mathbb{E}[<w_{t+\frac{1}{2}}^{l}-w^{*},\nabla f_{\mathcal{S}_{t}^{l}}(w_{t+\frac{1}{2}}^{l})>].
\end{split}
\end{equation}
According to the strongly convexity of the loss function,
\begin{equation}
\begin{split}
&\mathbb{E}[<w_{t+\frac{1}{2}}^{l}-w^{*},\nabla f_{\mathcal{S}_{t}^{l}}(w_{t+\frac{1}{2}}^{l})>]\\
& \geq \frac{\lambda L}{\lambda + L}\mathbb{E}[||w_{t+\frac{1}{2}}^{l}-w^{*}||^2] + \frac{1}{\lambda+L}||\nabla f_{\mathcal{S}_{t}^{l}}(w_{t+\frac{1}{2}}^{l})||^2.
\end{split}
\end{equation}
As a result,
\begin{equation}\label{lessthan0}
\begin{split}
&\mathbb{E}[||w_{t+\frac{1}{2}}^{l}-\eta\nabla f_{\mathcal{S}_{t}^{l}}(w_{t+\frac{1}{2}}^{l})-w^{*}||^2]\\
& \leq (1-\frac{2\eta\lambda L}{\lambda +L})\mathbb{E}[||w_{t+\frac{1}{2}}^{l}-w^{*}||^2] \\
&+ (\eta^2-\frac{2\eta}{\lambda+L})\mathbb{E}[||\nabla f_{\mathcal{S}_{t}^{l}}(w_{t+\frac{1}{2}}^{l})||^2].\\
\end{split}
\end{equation}

Note that when $\eta \leq \frac{2}{\lambda+L}$, the second term in (\ref{lessthan0}) is negative and therefore,
\begin{equation}
\begin{split}
&\mathbb{E}[||w_{t+\frac{1}{2}}^{l}-\eta\nabla f_{\mathcal{S}_{t}^{l}}(w_{t+\frac{1}{2}}^{l})-w^{*}||^2]\\
&\leq(1-\frac{2\eta\lambda L}{\lambda +L})\mathbb{E}[||w_{t+\frac{1}{2}}^{l}-w^{*}||^2].
\end{split}
\end{equation}
Since $\sqrt{1-2x} \leq 1-x$ when $1-2x > 0$, we have
\begin{equation}
\begin{split}
&\mathbb{E}[||w_{t+\frac{1}{2}}^{l}-\eta\nabla f_{\mathcal{S}_{t}^{l}}(w_{t+\frac{1}{2}}^{l})-w^{*}||]\\
&\leq(1-\frac{\eta\lambda L}{\lambda +L})\mathbb{E}[||w_{t+\frac{1}{2}}^{l}-w^{*}||].
\end{split}
\end{equation}
In addition, according to (\ref{condition2}),
\begin{equation}
\begin{split}
&\mathbb{E}[||w_{t+\frac{1}{2}}^{l}-w^{*}||]\\
&=\mathbb{E}[||\frac{w^{l}_{t}+\sum_{m \in \mathcal{A}_{t}^{l}}w^{m}_{t}}{|\mathcal{A}_{t}^{l}|+1}-w^{*}||] \\
&\leq \frac{1}{|\mathcal{A}_{t}^{l}|+1}\big[\mathbb{E}[||w^{l}_{t}-w^{*}||] + \sum_{m \in \mathcal{A}_{t}^{l}}\mathbb{E}[||w^{m}_{t}-w^{*}||]\big]\\
&\leq \mathbb{E}[||w^{l}_{t}-w^{*}||] + \frac{\delta}{t+1}.
\end{split}
\end{equation}
Then,
\begin{equation}
\begin{split}
&\sum_{l=1}^{N-p}\mathbb{E}[||w_{t+1}^{l}-w^{*}||] \\
&\leq \sum_{l=1}^{N-p}(1-\frac{\eta\lambda L}{(N-p)(\lambda+L)})\mathbb{E}[||w_{t}^{l}-w^{*}||]\\
&+(1-\frac{\eta\lambda L}{\lambda+L})\frac{\delta}{(N-p)(t+1)}.
\end{split}
\end{equation}
As a result,
\begin{equation}
\begin{split}
&\sum_{l=1}^{N-p}\mathbb{E}[||w_{t+1}^{l}-w^{*}||] \\
&\leq \sum_{l=1}^{N-p}(1-\frac{\eta\lambda L}{(N-p)(\lambda+L)})^{t}\mathbb{E}[||w_{0}^{l}-w^{*}||]\\
&+(1-\frac{\eta\lambda L}{\lambda+L})\frac{\delta}{(N-p)}\sum_{k=0}^{t}\frac{(1-\frac{\eta\lambda L}{(N-p)(\lambda+L)})^{(t-k)}}{k+1}.
\end{split}
\end{equation}
\end{proof}

\begin{algorithm}
\caption{Byzantine Tolerant SGD Algorithm when $p$ is unknown}
\label{algorithm2}
\begin{algorithmic}
\STATE 1. Initialization: total number of workers: $N$, number of training data samples for each node: $M$, upper bound of the number of Byzantine workers: $p$, each honest worker randomly initialize their model parameters $w_{t}^{i}$'s.
\STATE 2. for iteration $t=0,1,\cdots,T$ do
\STATE 3. ~~ if worker $i$ causes the master clock to tick:
\STATE 4. ~~~~worker $i$ sends requests and fetches the model parameters from all the other workers, and then accepts $w_t^{j}$ if it satisfies the following conditions:
\begin{equation}\label{condition2}
||w_{t}^{i}-w_{t}^{j}|| \leq \frac{\delta}{t+1}, \text{and}
\end{equation}
\begin{equation}\label{condition1}
<\nabla f_{\mathcal{S}_{t}^{i}}(w_{t}^{j}), w_{t}^{i}-w_{t}^{j}> \geq 0,
\end{equation}
in which $\nabla f_{\mathcal{S}_{t}^{i}}(w_{t}^{j})=\frac{1}{|\mathcal{S}_{t}^{i}|}\sum_{m \in \mathcal{S}_{t}^{i}}\nabla f_{m}(w_{t}^{j})$ is the average gradient corresponding to the shared model parameter $w_{t}^{j}$ and $\delta$ is subject to design and will be discussed in Section \ref{Numerical results}.
\STATE 5. ~~~~Then worker $i$ takes an average over the accepted model parameters, randomly samples a mini-batch of training samples $\mathcal{S}_{t}^{i}$ from its local dataset and performs one gradient descent step as follows:
\begin{equation}\label{TakeAverage}
  w_{t+\frac{1}{2}}^{i} = \frac{w^{i}_{t}+\sum_{j \in \mathcal{A}_{t}^{i}}w^{j}_{t}}{|\mathcal{A}_{t}^{i}|+1},
\end{equation}
\begin{equation}\label{PerformSGD}
  w_{t+1}^{i} = w_{t+\frac{1}{2}}^{i}-\eta_t\nabla f_{\mathcal{S}_{t}^{i}}(w_{t+\frac{1}{2}}^{i}),
\end{equation}
in which $\mathcal{A}_{t}^{i}$ is the set of accepted workers, $\nabla f_{\mathcal{S}_{t}^{i}}(w_{t+\frac{1}{2}}^{i})=\frac{1}{|\mathcal{S}_{t}^{i}|}\sum_{m \in \mathcal{S}_{t}^{i}}\nabla f_{m}(w_{t+\frac{1}{2}}^{i})$ is the average gradient and $\eta_t$ is the learning rate at time $t$.
\STATE 5. ~~~~worker $i$ normalizes its own model parameter, i.e.,
\begin{equation}\label{Normalization}
  w_{t+1}^{i} = \frac{w_{t+1}^{i}}{||w_{t+1}^{i}||}.
\end{equation}
\STATE 6.~~end if
\STATE 7.end for
\end{algorithmic}
\end{algorithm}

\begin{Remark}
According to (\ref{Contradiction}), the convergence of the local model parameters of the honest workers is immediate. In particular, the term induced by the bound given in (\ref{condition2}) decreases as the number of iterations increases and will finally vanish to 0. In addition, we note that although condition (\ref{condition2}) can guarantee the convergence of Algorithm \ref{algorithm2}, an appropriate $\delta$ should be determined for good performance. However, the choice of $\delta$ may depend on not only the specific dataset, but also the number of workers, which makes it hard to find a suitable $\delta$ in practice. This problem is solved by condition (\ref{condition1}), which preserves good performance even when we set $\delta$ arbitrarily large.
\end{Remark}

\begin{figure*}[h]%{r}{0.28\textwidth}
\begin{minipage}[t]{0.33\linewidth}
\centering
\includegraphics[width=1\textwidth]{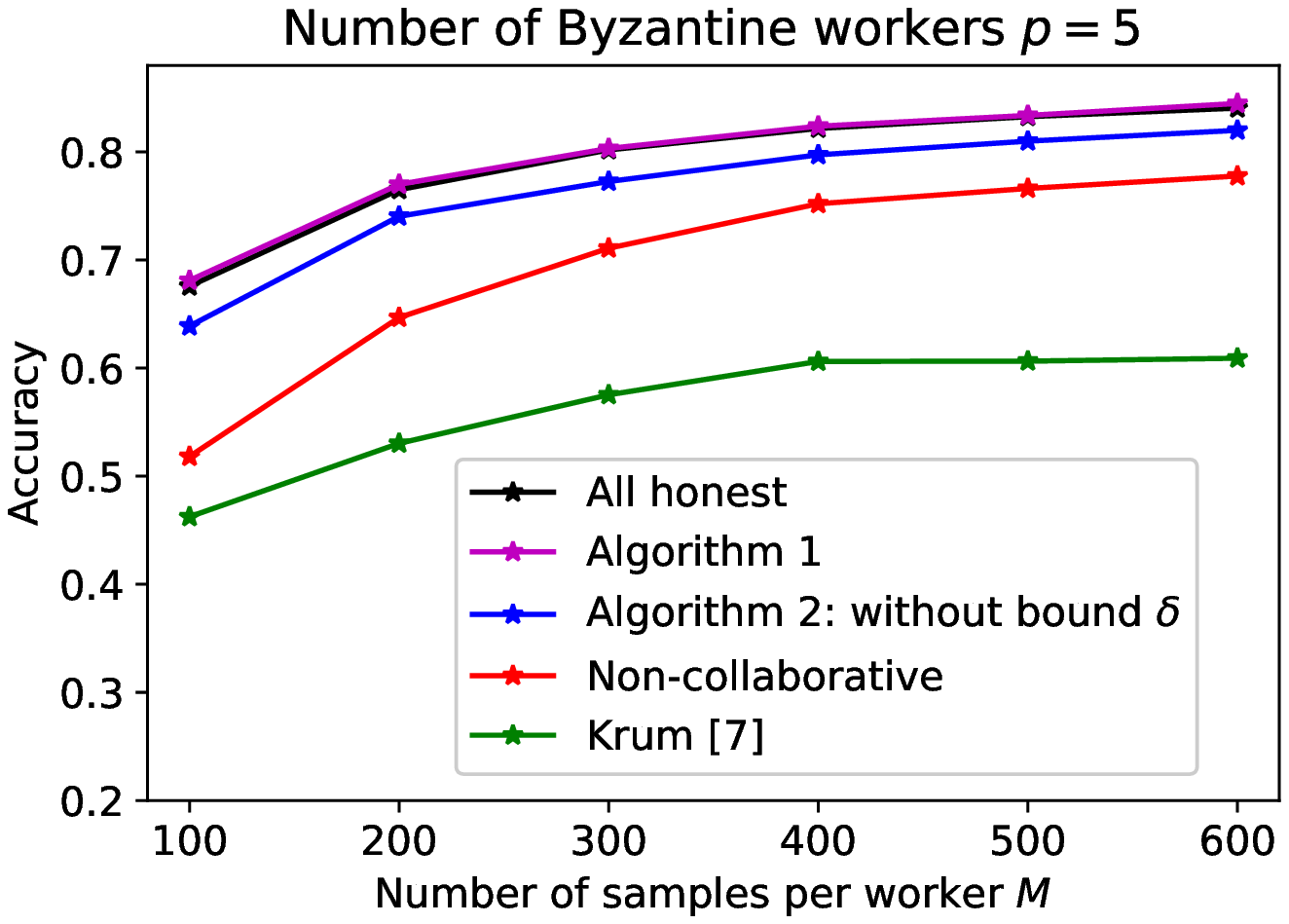}\vspace{-0.1in}
\caption{\footnotesize{The performance against ``Add noise" attack}}\vspace{-0.2in}
\label{addnoisep5}
\end{minipage}
\begin{minipage}[t]{0.33\linewidth}
\centering
\includegraphics[width=1\textwidth]{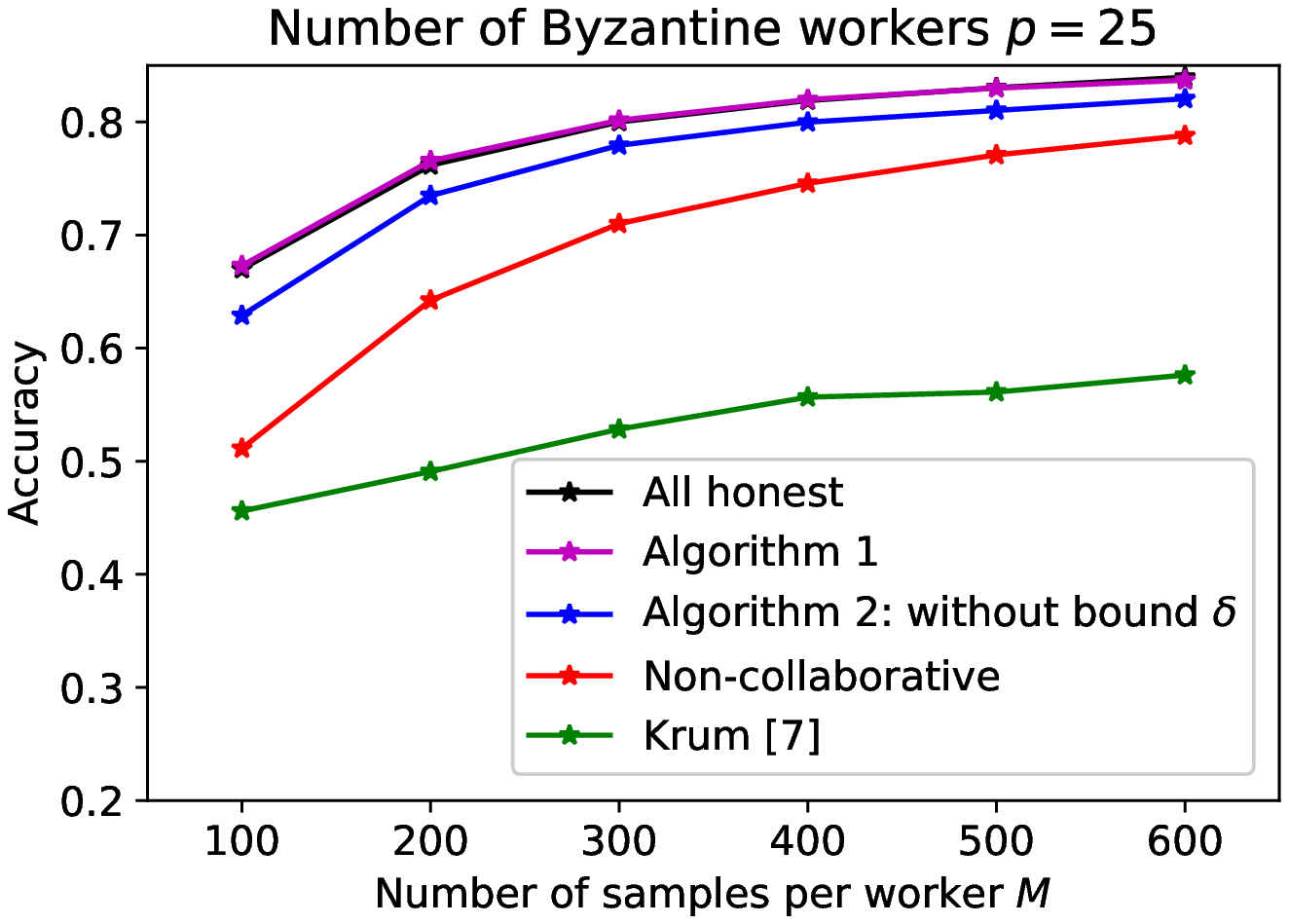}\vspace{-0.1in}
\caption{\footnotesize{The performance against ``Add noise" attack}}\vspace{-0.2in}
\label{addnoisep25}
\end{minipage}
\begin{minipage}[t]{0.33\linewidth}
\centering
\includegraphics[width=1\textwidth]{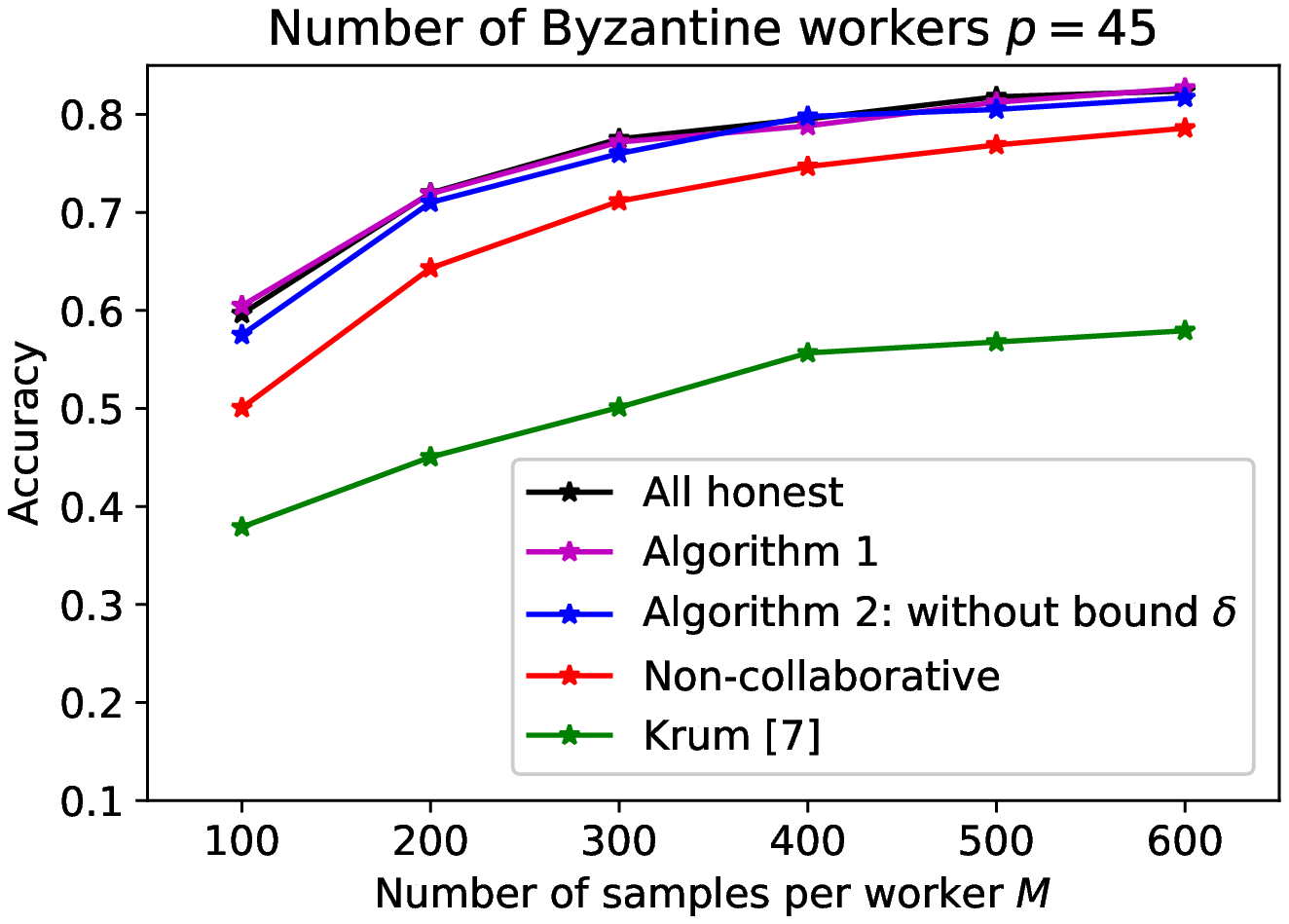}\vspace{-0.1in}
\caption{\footnotesize{The performance against ``Add noise" attack}}\vspace{-0.2in}
\label{addnoisep45}
\end{minipage}
\end{figure*}

\begin{figure*}[h]%{r}{0.28\textwidth}
\begin{minipage}[t]{0.33\linewidth}
\centering
\includegraphics[width=1\textwidth]{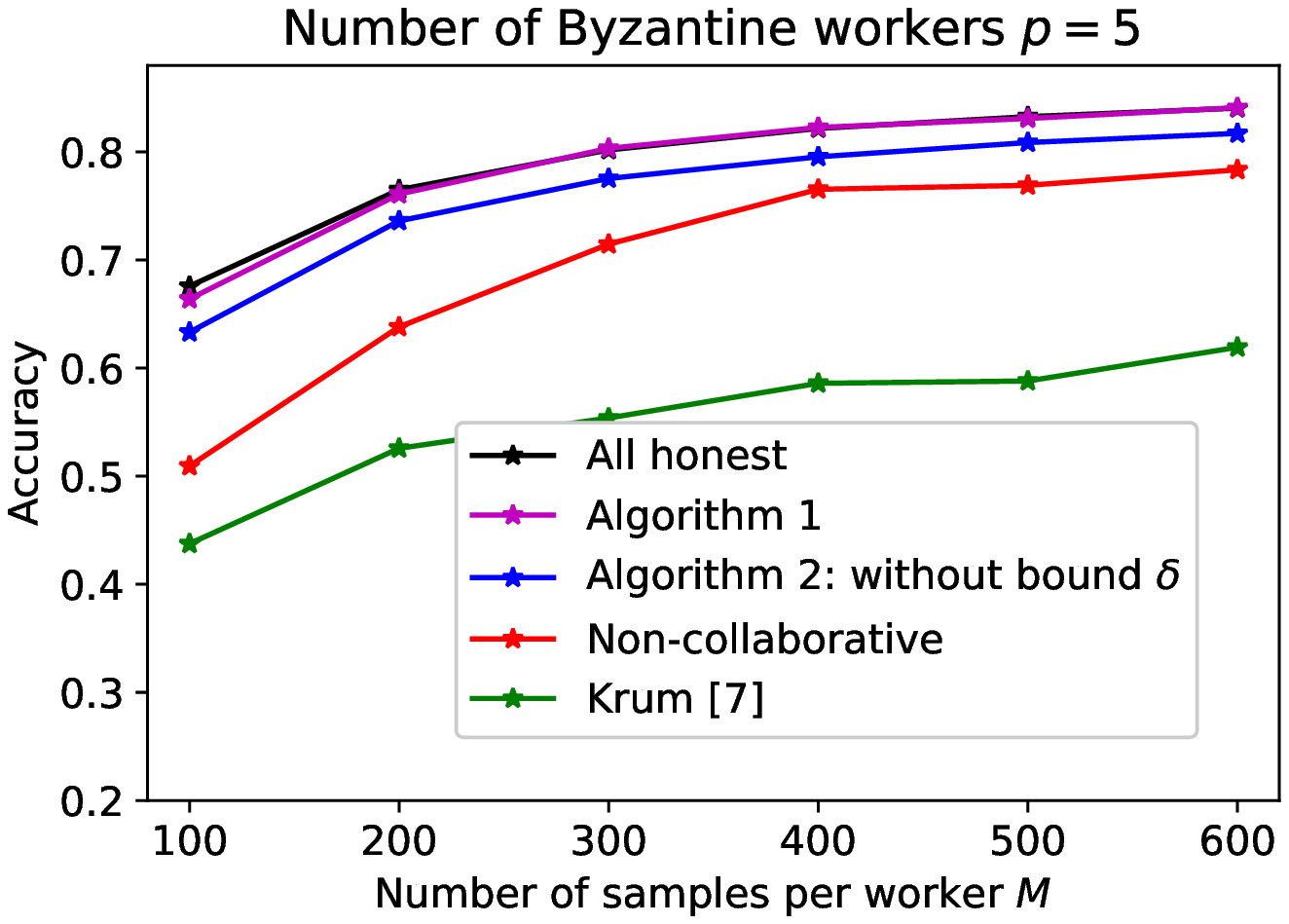}\vspace{-0.1in}
\caption{\footnotesize{The performance against ``Random" attack}}\vspace{-0.2in}
\label{Randomp5}
\end{minipage}
\begin{minipage}[t]{0.33\linewidth}
\centering
\includegraphics[width=1\textwidth]{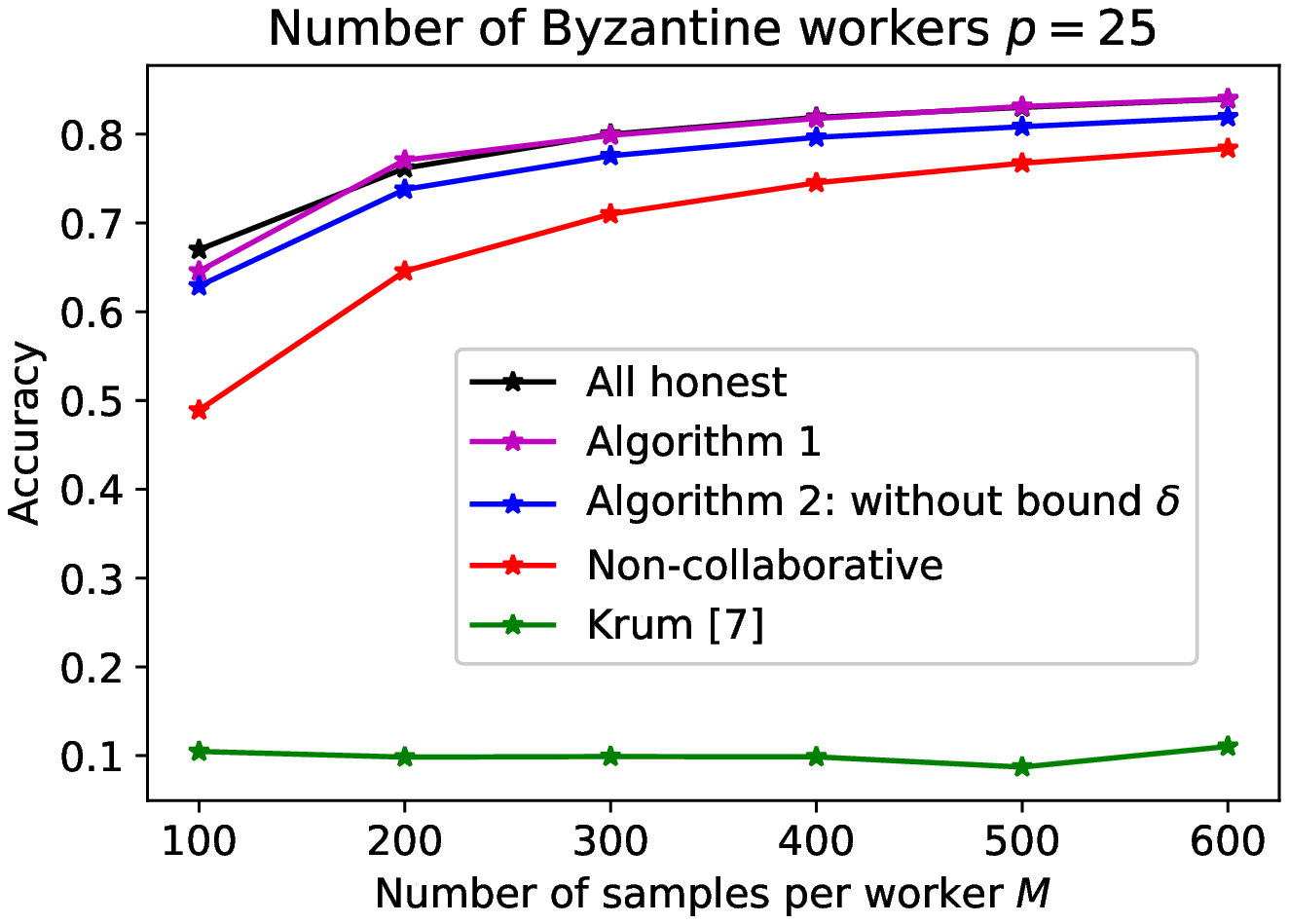}\vspace{-0.1in}
\caption{\footnotesize{The performance against ``Random" attack}}\vspace{-0.2in}
\label{Randomp25}
\end{minipage}
\begin{minipage}[t]{0.33\linewidth}
\centering
\includegraphics[width=1\textwidth]{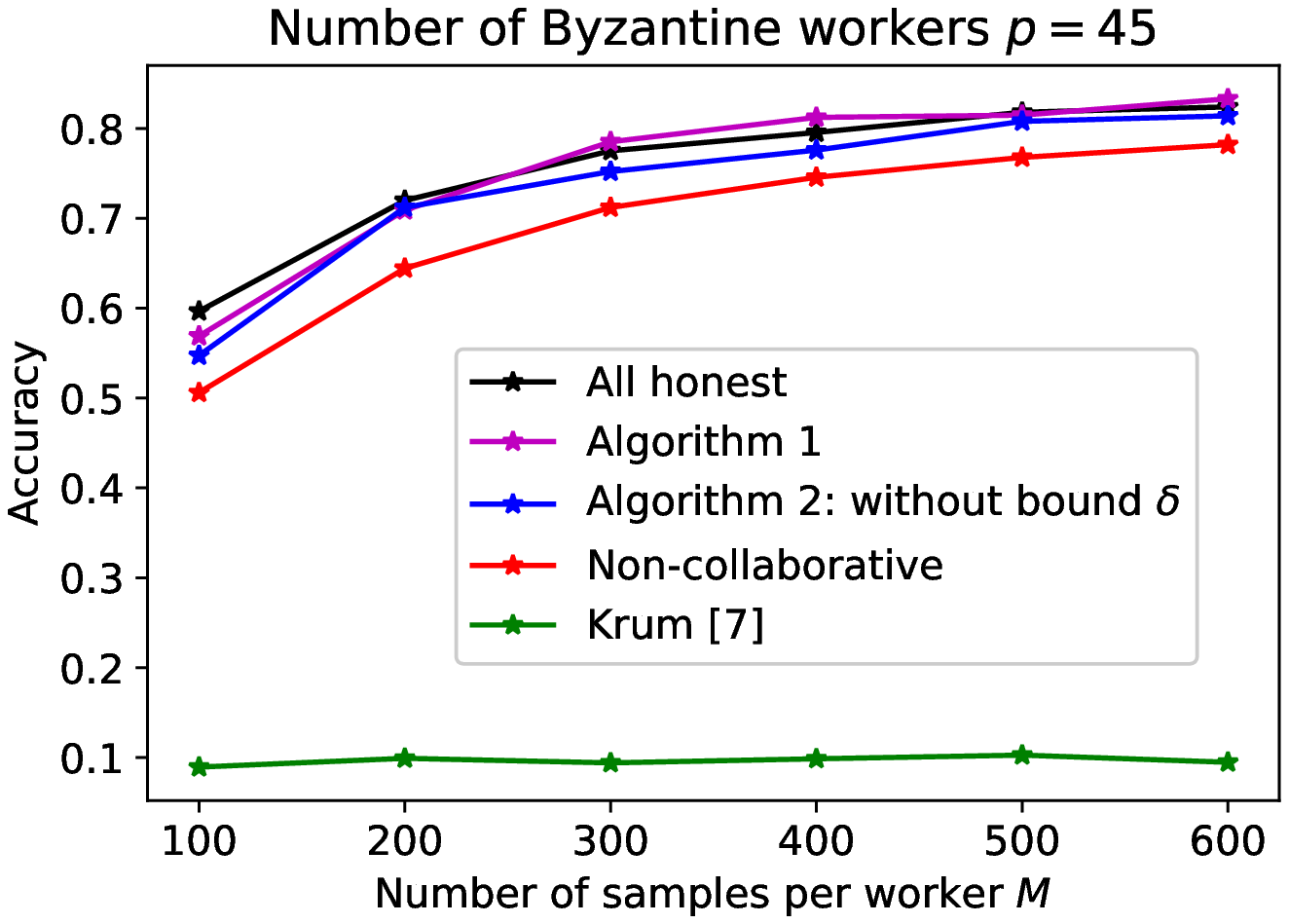}\vspace{-0.1in}
\caption{\footnotesize{The performance against ``Random" attack}}\vspace{-0.2in}
\label{Randomp45}
\end{minipage}
\end{figure*}

\section{Simulation Results}\label{Numerical results}
\noindent In this section, we present the simulation results to demonstrate the effectiveness of the proposed algorithms. In particular, the real public dataset MNIST \cite{lecun2010mnist} is used. MNIST is a widely used computer vision dataset which consists of 70,000 $28\times28$ pixel images of handwritten digits from 0 to 9. The dataset is divided into a training subset of size 60,000 and a testing subset of size 10,000. It is assumed that there are 50 workers in total (i.e., $N$=50) and all of them randomly select $M$ training samples from the training dataset and test their local model parameters using the testing subset after training. It is assumed that every worker builds a softmax regression model locally and runs one epoch (and therefore a larger local dataset results in more training iterations). During the training process, the local clocks of the workers are governed by the asynchronous model discussed in Section \ref{ProblemFormulation}. In addition, we assumed that the Byzantine workers train their own local model parameters independently and send requests to others but never use the information from others when they perform gradient descent steps.

To evaluate the effectiveness of the proposed algorithms, the average accuracy of the final local model parameters of all the legit workers is examined and compared with three baseline mechanisms. In the ``Non-collaborative" case, the honest workers independently train their local model parameters and do not collaborate at all; in the ``All honest" case, all the workers are supposed to be honest; for the baseline ``Krum", we implement the algorithm proposed in \cite{blanchard2017machine}. In addition, three types of attacks are considered. In the ``Add noise" attack, the Byzantine attackers add a random Gaussian noise with zero mean and a variance of 0.1 to the local model parameters of the workers that send requests; in the ``Random" attack, the Byzantine attackers generate and share a random vector with each element drawn from a uniform distribution in [0,1]; in the ``Inverse" attack, the Byzantine attackers share the opposite value of their own local model parameters.

\subsection{The Performance of the Proposed Algorithms against Different Attacks}
\noindent In this subsection, the performance of the proposed algorithms against different attacks is examined. In particular, in the ``All honest" case, it is assumed that the total number of honest workers is the same as the other examined mechanisms (i.e., if $p=5$, then there are $N-p=45$ workers). For the implementation of Algorithm \ref{algorithm2}, we present the results ignoring (\ref{condition2}) and the choice of $\delta$ will be discussed in Section \ref{deltachoice}. It can be observed from Figs. \ref{addnoisep5}-\ref{addnoisep45} that both Algorithm \ref{algorithm1} and Algorithm \ref{algorithm2} perform better than the ``Non-collaborative" case and the ``Krum" counterpart against the "Add noise" attack. In particular, ``Krum" performs even worse than the ``Non-collaborative" case since it only utilizes one of the gradients shared by all the workers and therefore discards useful information from most of the legit workers. In addition, for Algorithm \ref{algorithm1}, it is assumed that the exact number of Byzantine workers is known, and therefore it achieves almost the same performance as that in the ``All honest" counterpart, which can be considered as the optimal case. When $p$ is small, Fig. \ref{addnoisep5} and Fig. \ref{addnoisep25} show that Algorithm \ref{algorithm2} is about 2\% worse than Algorithm \ref{algorithm1} in terms of testing accuracy. This is because the condition given in (\ref{condition1}) may filter out some useful information from the honest workers. However, when the number of Byzantine workers is large (i.e., $p=45$ in the simulation), the performance of Algorithm \ref{algorithm2} is comparable to that of Algorithm \ref{algorithm1} and the ``All honest" counterpart.

Figs. \ref{Randomp5}-\ref{Randomp45} and Figs. \ref{Inversep5}-\ref{Inversep45} show the performance of the proposed algorithms against the ``Random" attack and the ``Inverse" attack, respectively. It can be observed the proposed algorithms outperform the ``Non-collaborative" and ``Krum" counterparts, which further verifies their effectiveness.

\begin{figure*}[h]%{r}{0.28\textwidth}
\begin{minipage}[t]{0.33\linewidth}
\centering
\includegraphics[width=1\textwidth]{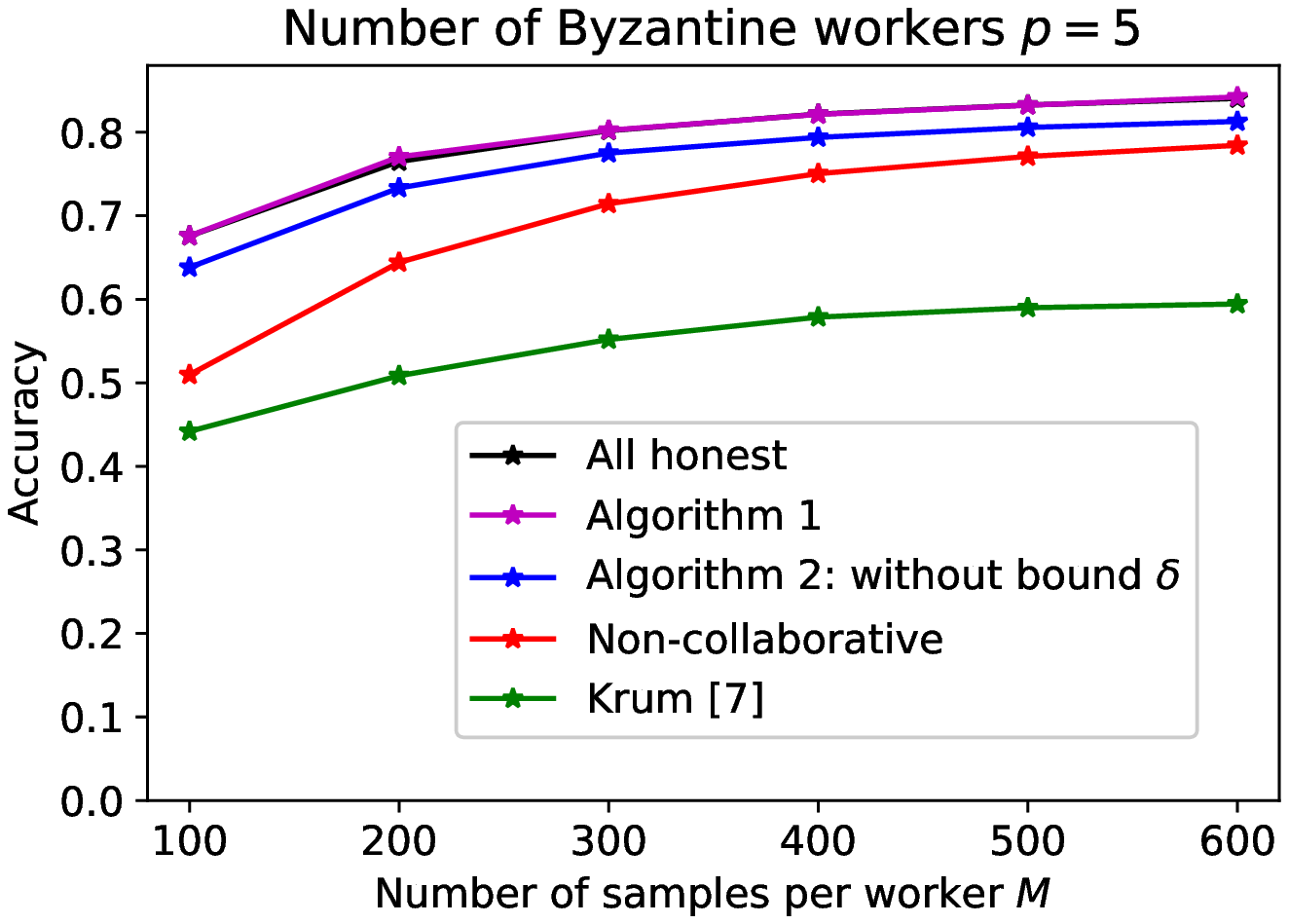}\vspace{-0.1in}
\caption{\footnotesize{The performance against ``Inverse" attack}}\vspace{-0.2in}
\label{Inversep5}
\end{minipage}
\begin{minipage}[t]{0.33\linewidth}
\centering
\includegraphics[width=1\textwidth]{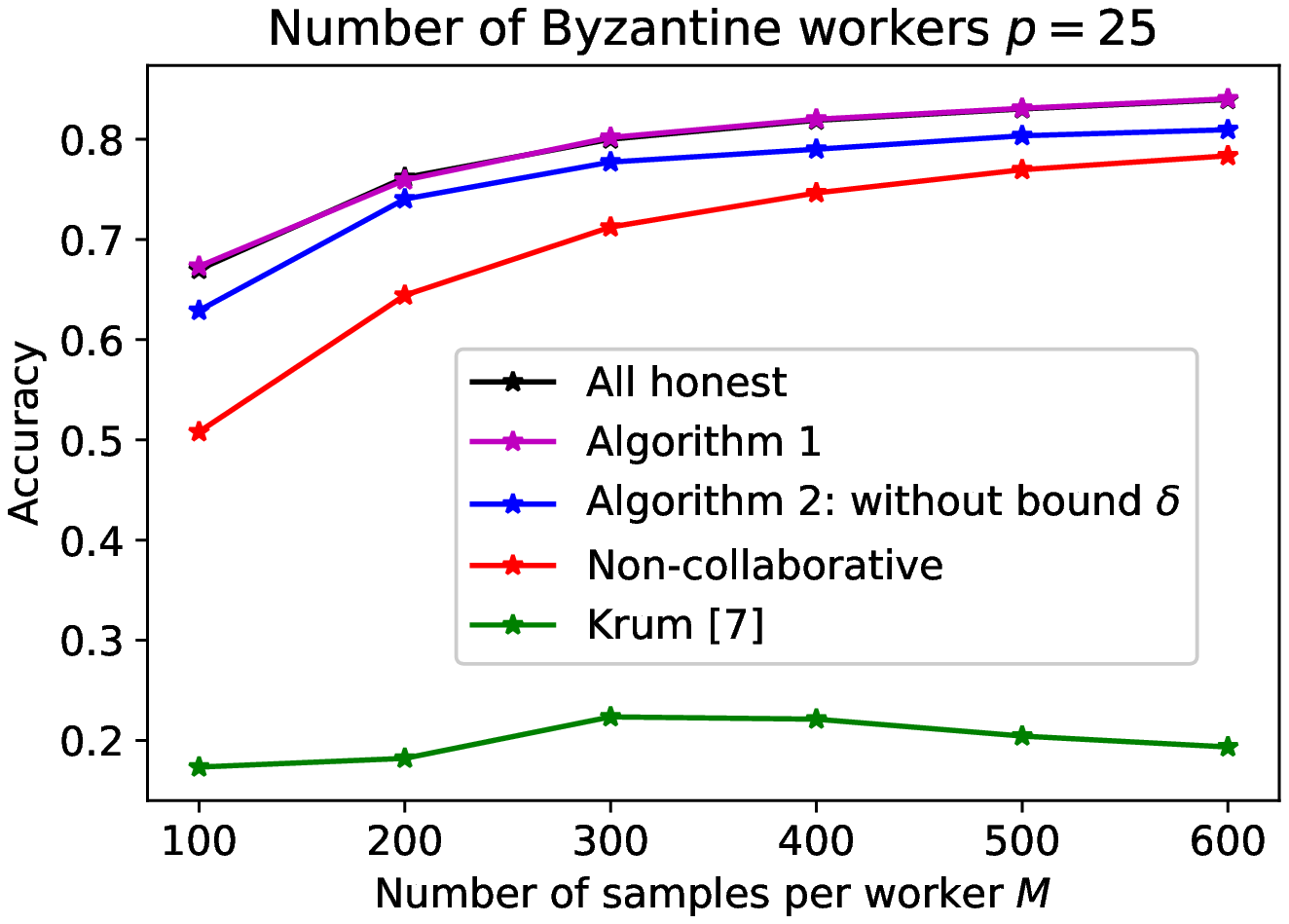}\vspace{-0.1in}
\caption{\footnotesize{The performance against ``Inverse" attack}}\vspace{-0.2in}
\label{Inversep25}
\end{minipage}
\begin{minipage}[t]{0.33\linewidth}
\centering
\includegraphics[width=1\textwidth]{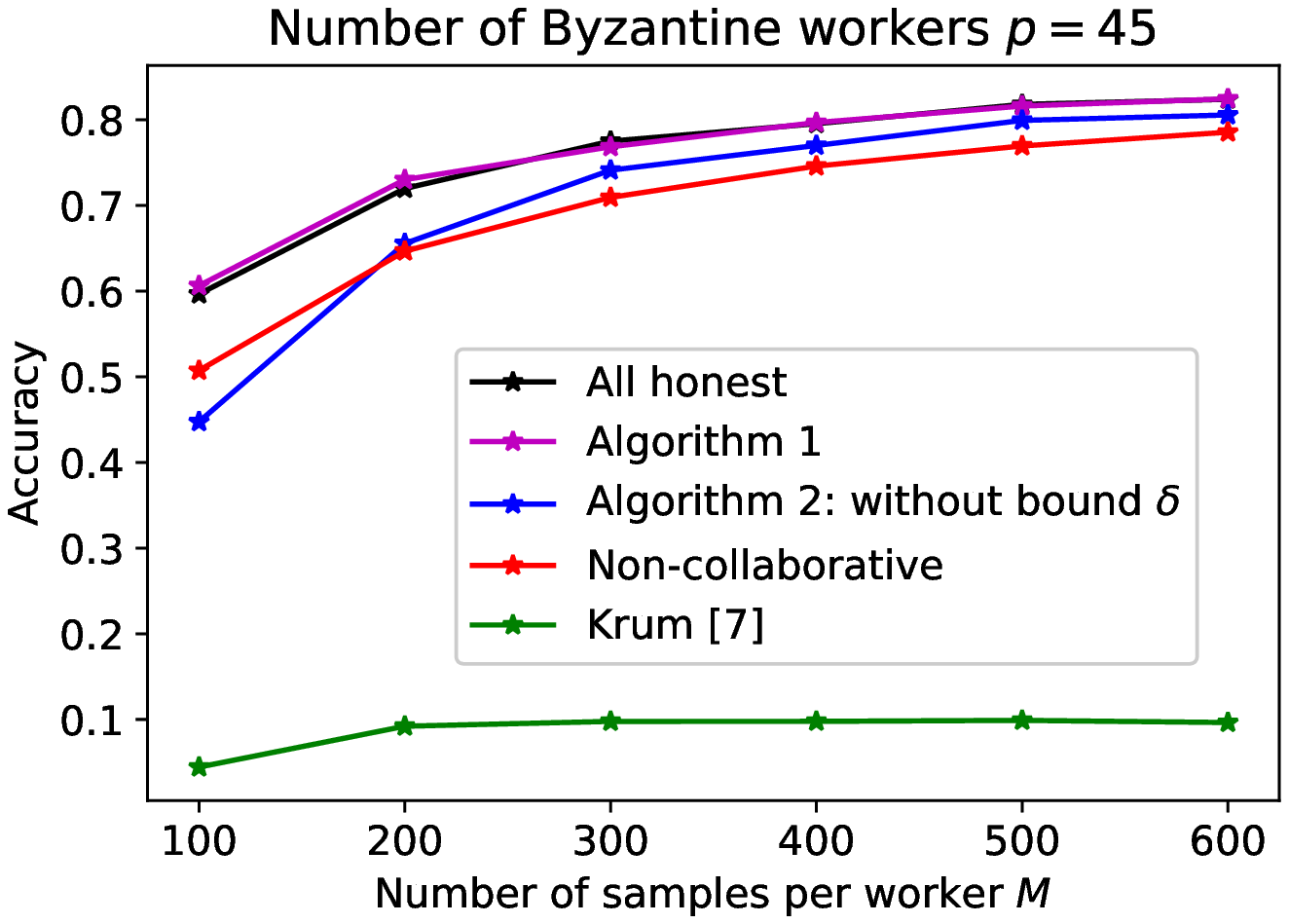}\vspace{-0.1in}
\caption{\footnotesize{The performance against ``Inverse" attack}}\vspace{-0.2in}
\label{Inversep45}
\end{minipage}
\end{figure*}
\begin{figure*}[h]%{r}{0.28\textwidth}
\begin{minipage}[t]{0.33\linewidth}
\centering
\includegraphics[width=1\textwidth]{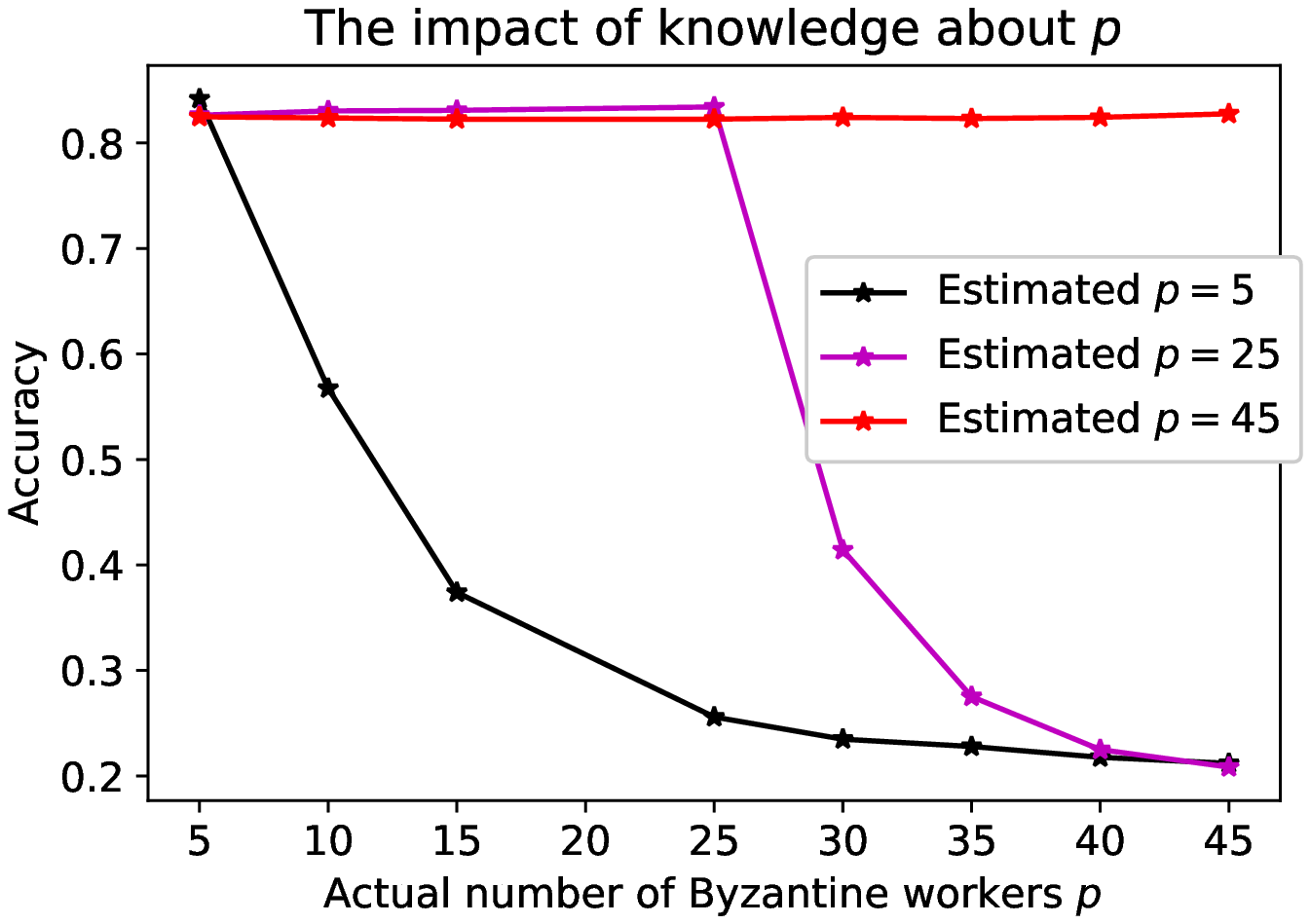}\vspace{-0.1in}
\caption{\footnotesize{The impact of knowledge about $p$ on Algorithm 1}}\vspace{-0.2in}
\label{addnoisep}
\end{minipage}
\begin{minipage}[t]{0.33\linewidth}
\centering
\includegraphics[width=1\textwidth]{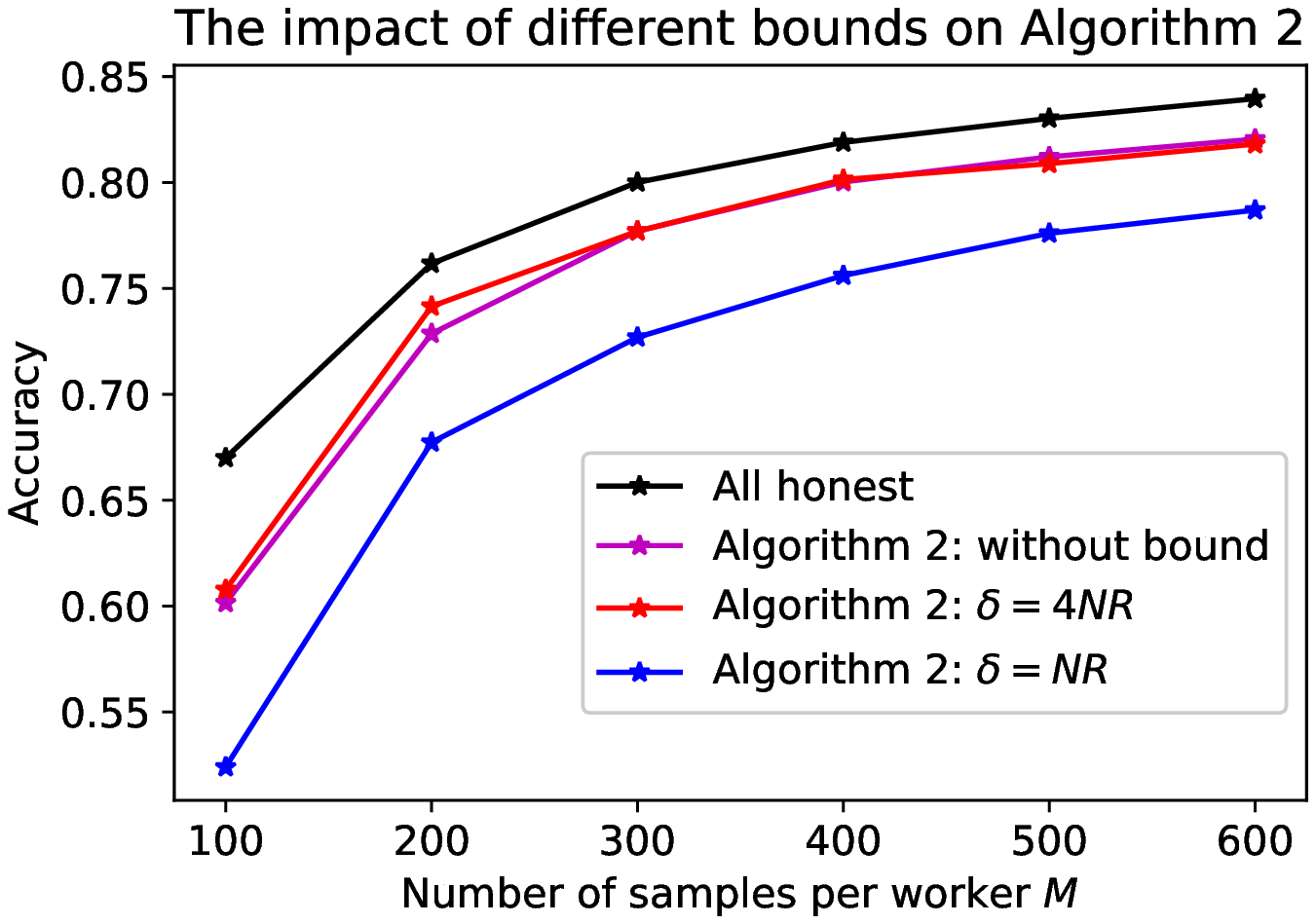}\vspace{-0.1in}
\caption{\footnotesize{The impact of the bound $\delta$ on Algorithm 2}}\vspace{-0.2in}
\label{addnoisedelta1}
\end{minipage}
\begin{minipage}[t]{0.33\linewidth}
\centering
\includegraphics[width=1\textwidth]{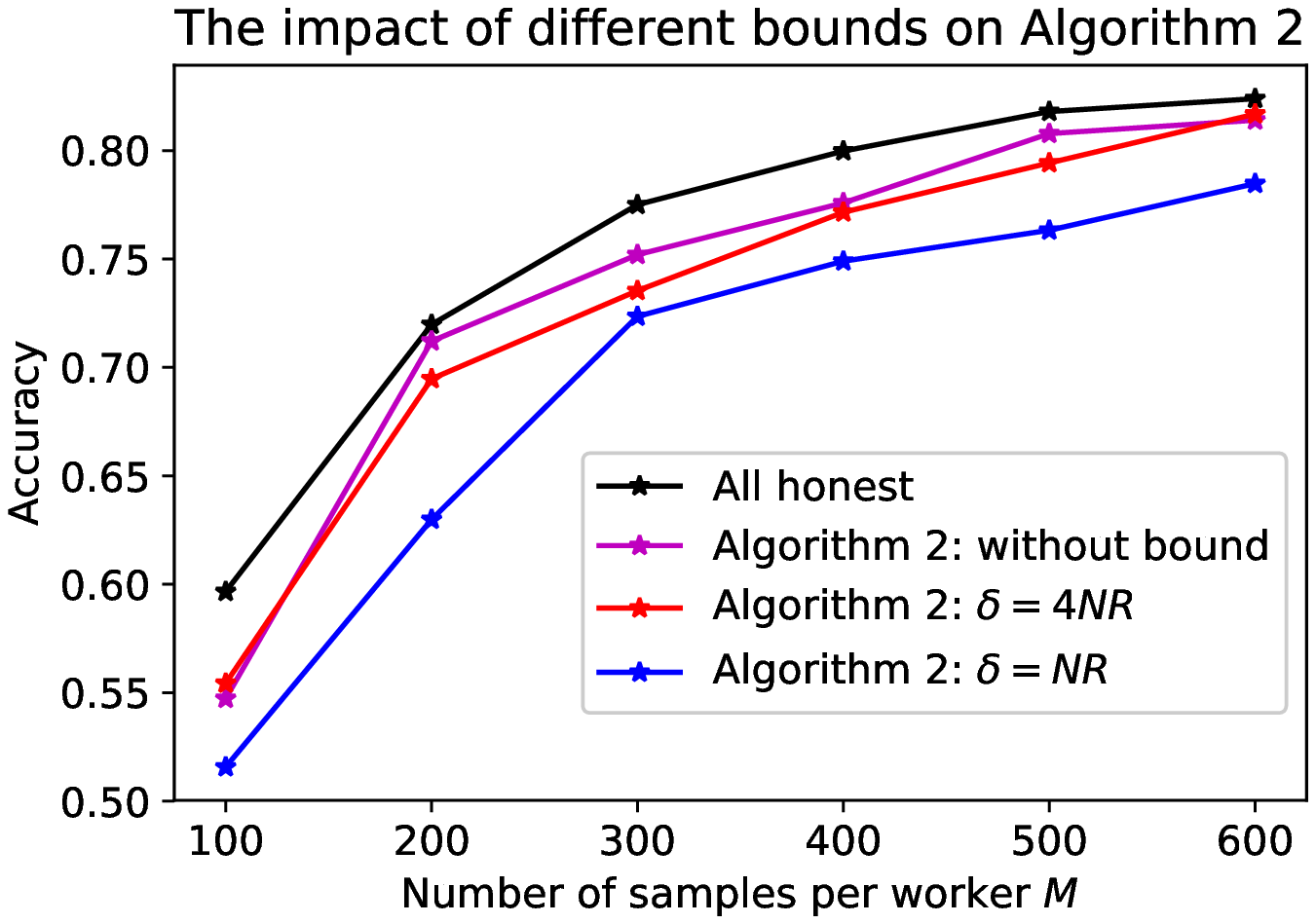}\vspace{-0.1in}
\caption{\footnotesize{The impact of the bound $\delta$ on Algorithm 2}}\vspace{-0.2in}
\label{addnoisedelta2}
\end{minipage}
\end{figure*}

\subsection{The Impact of Knowledge about $p$ on Algorithm 1}
\noindent Note that in the previous discussion, it is assumed that the exact number of Byzantine workers is known in the implementation of Algorithm \ref{algorithm1}. In practice, however, such an assumption is rarely valid. Fig. \ref{addnoisep} shows the performance of Algorithm \ref{algorithm1} under ``Random" attack when the estimated number of Byzantine workers (i.e., the $p$ used in Algorithm \ref{algorithm1}) is different from the actual one, in the case that each worker has 600 training samples. It can be seen that when the actual number of Byzantine workers exceeds the estimated one, the performance of Algorithm \ref{algorithm1} degrades quickly. On the other hand, an accurate estimated $p$ can lead to better performance. For example, Algorithm \ref{algorithm1} with an estimated $p=25$ performs better than the estimated $p=45$ counterpart when there are less than 25 Byzantine workers. We note that when the number of local training samples is large enough, assuming a large $p$ does not degrade the performance much since the workers can afford to discard useful information from some of the legit workers. However, an accurate estimate of $p$ can be essential when the workers have only a limited number of training samples.

\subsection{The Impact of Different Bound $\delta$ on Algorithm 2}\label{deltachoice}
\noindent Fig. \ref{addnoisedelta1} and Fig. \ref{addnoisedelta2} show the performance of Algorithm \ref{algorithm2} under ``Random" attack with different bound $\delta$, in the case that there are 25 and 45 Byzantine attackers respectively. In particular, $N$ is the total number of workers and $R$ satisfies $||w|| \leq R, \forall w\in\mathcal{W}$. It can be seen that a smaller bound may lead to worse performance since more useful information may be filtered out. In fact, the optimal choice of $\delta$ may depend on the specific datasets, attacks and the number of Byzantine workers and it can be computationally expensive to obtain. However, Fig. \ref{addnoisedelta1} and Fig. \ref{addnoisedelta2} show that if we remove the bound $\delta$ (or equivalently set $\delta$ to arbitrarily large), the performance of Algorithm \ref{algorithm2} is only around 2\% worse than the ``All honest" case.\footnote{Similar results can be observed for other scenarios and are omitted in the interest of space.}
%\begin{figure}%{r}{0.28\textwidth}
%\centering
%\includegraphics[width=0.45\textwidth]{Addnoise7.eps}
%\caption{\footnotesize{The impact of the bound $\delta$ on Algorithm 2}}\vspace{-0.2in}
%\label{addnoisedelta}
%\end{figure}

\section{Related Works}\label{Related Works}
\noindent There have been many prior works on Byzantine tolerant SGD algorithms. In particular, \cite{chen2017distributed} proposes a geometric median based aggregation rule to calculate the gradient used for parameter update, given all the gradients received from the workers. In \cite{blanchard2017machine}, given the total number of workers $N$ and the number of Byzantine workers $p$, for each worker $i$ and its gradient $\nabla f^{i}$, the parameter server first selects a set $V_i$ that contains the $N-p-2$ closest gradients to $\nabla f^{i}$. Then a score $s_i$ is computed for each worker $i$, which measures how close its gradient is to the gradients in $V_i$ (i.e., $s_i=\sum_{j\in V_{i}}||\nabla f^{i}-\nabla f^{j}||^2$). Finally, the worker with the minimum score is selected and its gradient is used for parameter update. \cite{xie2018generalized} considers generalized Byzantine attackers which attack certain elements of the gradient vectors instead of the whole gradient vectors and proposes modified median based aggregation rules. \cite{yin2018byzantine} proposes coordinate-wise median and coordinate-wise trimmed mean based aggregation rules for gradient selection. \cite{alistarh2018byzantine} tries to identify the good workers by comparing their shared gradients with the medians and use the gradient information from the good workers for parameter update. However, the algorithms proposed in \cite{chen2017distributed,blanchard2017machine,xie2018generalized,yin2018byzantine,alistarh2018byzantine} become incompetent when more than half of the workers are Byzantine. In addition, synchronous settings are assumed (i.e., the workers with better computation capability have to wait for the other slower workers) in these works, which leads to waste of computation resources. \cite{damaskinos2018asynchronous} proposes an asynchronous Byzantine tolerant SGD algorithm. Particularly, it consists of a Byzantine-resilient filter and a frequency filter to determine whether a (possibly outdated) gradient should be accepted or not. However, it can only deal with up to $\frac{1}{3}$ Byzantine workers. In \cite{caodistributed}, it is assumed that the parameter server has a small portion of dataset locally, which is used to compute a noisy version of the true gradient. After receiving the gradients from the workers, the parameter server compares them with the local noisy gradients and decides to accept them if the difference is within a threshold. In this sense, the algorithm proposed in \cite{caodistributed} can deal with an arbitrary number of Byzantine workers and therefore is the most relevant one to this work. However, it still requires a parameter server to collect the gradients and therefore may be vulnerable to the single point of failure. In addition, it requirs to manually set the threshold, which depends on specific datasets. Finally, only synchronous scenarios are considered in \cite{caodistributed}.
%
%\textcolor{blue}{A summary of the comparison of the proposed algorithms with previous Byzantine tolerant algorithms is given in Table \ref{table1}.}

%\begin{table}[!t]
%\color{blue}
%\centering
%\caption{\footnotesize{Comparison with previous Byzantine tolerant algorithms}}
%\label{table1}
%\begin{tabular}{ | m{1.4cm} | m{1cm}| m{1cm} | m{1.7cm} | m{1.5cm} |}
%\hline
%       & Maximum $p$ & $p$ known? & Asynchronous? & Parameter server required?\\
%\hline
%\cite{chen2017distributed,blanchard2017machine,xie2018generalized,yin2018byzantine}   & $\frac{N}{2}$ & Yes & No & Yes\\
%\hline
%\cite{alistarh2018byzantine}    & $\frac{N}{2}$ & No & No& Yes\\
%\hline
%\cite{damaskinos2018asynchronous}  & $\frac{N}{3}$ & Yes & Yes& Yes\\
%\hline
%\cite{caodistributed}    & N & No & No& Yes\\
%\hline
%Algorithm 1 & N & Yes & Yes& No\\
%\hline
%Algorithm 2 & N & No & Yes& No\\
%\hline
%\end{tabular}
%\end{table}

\section{Conclusions and Future Works}\label{Conclusions and Future Works}
\noindent Considering that most of the Byzantine tolerant SGD algorithms in the literature are either synchronous or robust against a limited number of Byzantine workers, two asynchronous distributed Byzantine tolerant SGD algorithms that can deal with an arbitrary number of Byzantine workers are proposed in this work. The convergence analysis for both algorithms is provided and the simulation results show that the proposed algorithms work well against all types of the examined Byzantine attacks. Since the proposed algorithms only consider the current shared information to decide whether to accept them or not, considering the usage of past information for performance improvement remains our future work.

\bibliography{Ref-Richeng}
\bibliographystyle{IEEEtran} %% plain.bst

\end{document}